\documentclass[letterpaper, 10 pt, conference]{ieeeconf}

\IEEEoverridecommandlockouts                             

\usepackage{graphics} 
\usepackage{epsfig} 
\usepackage{amsmath} 
\usepackage{amssymb}

\usepackage[ruled, linesnumbered]{algorithm2e}
\usepackage{subfigure}
\usepackage{stfloats}
\usepackage{adjustbox}
\newtheorem{lemma}{Lemma} 
\newtheorem{Theorem}{Theorem} 
\usepackage{cite}

\usepackage{hyperref}
\hypersetup{
    colorlinks=true,
    linkcolor=blue,      
    urlcolor=blue
}

\newtheorem{thm}{Theorem}

\newtheorem{prop}[thm]{Proposition}

\newtheorem{rem}{Remark}

\usepackage{xcolor}

\usepackage{soul}

\title{\LARGE \bf
Safety-Critical Control with Uncertainty Quantification using \\Adaptive Conformal Prediction
}

\author{Hao Zhou, Yanze Zhang, and Wenhao Luo
\thanks{$^*$This work was supported in part by the Faculty Research Grant award at UNC Charlotte, and in part by the U.S. National Science Foundation under Grant CNS-2312465.}
\thanks{Authors are with the Department of Computer Science, University of North Carolina at Charlotte, Charlotte, NC 28223, USA.
Email: {\tt\small \{hzhou11, yzhang94, wenhao.luo\}@uncc.edu}
}
}

\begin{document}

\maketitle
\thispagestyle{empty}
\pagestyle{empty}

\begin{abstract}
Safety assurance is critical in the planning and control of robotic systems. For robots operating in the real world, the safety-critical design often needs to explicitly address uncertainties and the pre-computed guarantees often rely on the assumption of the particular distribution of the uncertainty. However, it is difficult to characterize the actual uncertainty distribution beforehand and thus the established safety guarantee may be violated due to possible distribution mismatch. In this paper, we propose a novel safe control framework that provides a high-probability safety guarantee for stochastic dynamical systems following unknown distributions of motion noise. Specifically, this framework adopts adaptive conformal prediction to dynamically quantify the prediction uncertainty from online observations and combines that with the probabilistic extension of the control barrier functions (CBFs) to characterize the uncertainty-aware control constraints. By integrating the constraints in the model predictive control scheme, it allows robots to adaptively capture the true prediction uncertainty online in a distribution-free setting and enjoys formally provable high-probability safety assurance. Simulation results on multi-robot systems with stochastic single-integrator dynamics and unicycle dynamics are provided to demonstrate the effectiveness of our framework.

\end{abstract}

\section{Introduction}\label{sec: introduction}
Safety-critical control plays a key role in many robotic applications, such as multi-robot navigation \cite{wang2017safety, lyu2023decentralized} and autonomous driving \cite{lyu2021probabilistic}.
Model-based approaches such as the 
control barrier functions (CBFs) \cite{ames2019control, ames2016control} have been widely applied to enforce formally provable safety for deterministic dynamical systems.
However, realistic factors such as uncertainty, non-determinism, and lack of complete information 
often make it challenging to provide safety assurance for those stochastic dynamical systems in the real world.

Existing works on uncertainty-aware safety-critical control primarily focus on incorporating assumptions regarding specific uncertainty distributions or adopting conservative bounds. Examples include the Gaussian representation \cite{sadigh2016safe, zhu2019chance}, the bounded uniform distribution \cite{luo2020multi}, and conservative bounding volumes \cite{park2018fast}. Such approaches enable the explicit modeling of uncertainty's impact on ensuring safety which can thus be used for safe control designs.
For instance, model predictive path integral (MPPI) \cite{williams2015model, williams2017information} as a model-based method has been employed in various applications to account for constrained optimal control for stochastic dynamical systems. 
It generates sampling trajectories parallelly based on the assumed Gaussian distribution of the control policy and then derives the near-optimal control policy.
To address unmodeled uncertainties, safe learning techniques have been proposed \cite{ luo2022sample, wang2017safe} to heuristically quantify the uncertainty during exploration and combine it with control theoretic approaches for safe operations.
However, the safety guarantee derived in these works heavily relies on the accuracy of the assumed or learned uncertainty distribution, and it remains challenging for robots operating in more realistic scenarios where the uncertainty distributions could be arbitrary and difficult to learn.

Recently, statistical techniques such as conformal prediction (CP) \cite{lei2015distribution, papadopoulos2002inductive} have been prevalent 
given their ability to 
quantify the prediction uncertainty 
of the employed machine learning or dynamics models. CP leverages a sequence of calibration data to create prediction sets that are likely to cover the true value 
with high probability. However, the assumption of static data distribution does not apply to
many robotic applications where robots could often operate in uncertain and dynamic environments.
To address this limitation, the concept of Adaptive Conformal Prediction (ACP) has been introduced in \cite{gibbs2021adaptive, gibbs2022conformal} that can adaptively adjust the uncertainty quantification to maintain prediction accuracy under distribution shift. This enables the system to address safety-critical control when the underlying uncertainty distribution of the environment is not static but evolves over time. 
For example, work in \cite{dixit2023adaptive} adopts the distance-based condition between the robot and the obstacles as constraints quantified by the ACP and embeds them to the model predictive control (MPC) framework to produce the collision-free path. However, as suggested in \cite{zeng2021safety} the distance-based safety constraint in the MPC framework may require a longer horizon for effective collision avoidance and thus unnecessarily increase the computational time. 

In this paper, 
we propose a novel uncertainty-aware safe control framework that integrates the probabilistic CBFs with adaptive conformal prediction for stochastic dynamical system with unknown motion noise. The framework is able to provably quantify the uncertainty of future robot states predicted by the stochastic system dynamics model,
and utilize them to design CBF-based control constraints for high-probabilty safety guarantee.
The \textbf{contributions} in this paper are threefold:
\begin{itemize}
    \item We propose a novel provable probabilistic safe barrier certificate with Adaptive Conformal Prediction (ACP-SBC) for safe control under unknown motion noise.
    \item Theoretical analysis are provided to justify the enforced high-probability safety guarantee of our proposed method, where the derived guarantee under quantified uncertainty is distribution-free, i.e. it does not require knowing the distribution of motion noise as a prior.
    \item Simulation results on multi-robot systems are provided to evaluate the effectiveness of the proposed framework.
\end{itemize}

\section{Preliminaries and Problem formulation}\label{Sec: prelim and prob}
\subsection{Control Barrier Functions} \label{sec: CBFs}
Consider the robot dynamics as the following control affine system, 
\begin{equation}
\label{eq: contol affine system} 
    \dot{x} = f(x) + g(x)u
\end{equation}
where $x \in \mathbb{R}^{n}$ is the robot state and $u \in \mathbb{R}^m$ is the control input. $f:\mathbb{R}^n \mapsto \mathbb{R}^n$ and $g:\mathbb{R}^{n} \mapsto {\mathbb{R}^{n \times m}}$ are locally Lipschitz continuous.

Control barrier functions (CBFs) \cite{ames2019control} are often designed to ensure the robot's safety. 
The set of robot states satisfying the safety constraints can be denoted by $\mathcal{H}$ and expressed as the zero-superlevel set of a continuously differentiable function $h(x): \mathbb{R}^{n} \mapsto \mathbb{R}$,
\begin{align}
    \mathcal{H}&=\left \{x \in \mathbb{R}^n | h(x)\geqslant 0 \right \}, \notag \\
    \partial \mathcal{H}&=\left \{x \in \mathbb{R}^n | h(x) = 0 \right \} \label{eq: safe constraint}, \\
    \mathrm {Int}(\mathcal {H})&=\left \{x \in \mathbb{R}^n | h(x) > 0 \right \}. \notag
\end{align}
where $ \partial \mathcal{H}$ and $\mathrm {Int}(\mathcal {H})$ define the boundary and interior of the safe set $\mathcal{H}$, respectively.

\begin{lemma}\label{def: CBFs}
\textbf{CBFs.} [Summarized from \cite{ames2019control}]
Given the system dynamics~Eq.(\ref{eq: contol affine system}) affine in control and the safe set $\mathcal{H}$ as the 0-super level set of a continuously differentiable function $h(x): \mathbb{R}^n \mapsto \mathbb{R}$, the function $h$ is called a control barrier function if there exists an extended class-$\mathcal{K}$ function $\mathcal{K}(\cdot)$, such that 
$\sup_{u\in \mathbb{R}^m}\{\dot{h}(x, u)\}\geq -\mathcal{K}(h(x))$ for all $x \in \mathbb{R}^n$. The admissible control space for any Lipschitz continuous controller $u \in \mathbb{R}^m$ rendering $\mathcal{H}$ forward invariant (i.e., keeping the system state $x$ staying in $\mathcal{H}$ overtime) thus becomes:
{\small\begin{align}\label{eq: safe control space}
    \mathcal{B}(x) = \left \{u \in \mathcal{U} | L_fh(x) +L_gh(x)u + \mathcal{K} (h(x)) \geqslant 0 \right\}
\end{align}} where $L_fh$ and $L_gh$ are the Lie derivatives of $h$ along the function $f$ and $g$ respectively. 
The extended class $\mathcal{K}$ function can be commonly chosen as $\mathcal{K} (h(x))=\gamma h(x)$ with $\gamma >0$ as in \cite{luo2020multi}.
The condition of $B(x,u)=L_fh(x) +L_gh(x)u + \mathcal{K} (h(x)) \geqslant 0$ in Eq.(\ref{eq: safe control space}) can be used to denote the safety barrier certificates (SBC) \cite{wang2017safety} that define satisfying controller $u$ for the robot to stay within the safe set $\mathcal{H}$. 
\end{lemma}

\subsection{Chance-Constrained Safety}
Even though the Lemma~\ref{def: CBFs} reveals the explicit condition for robot safety with the perfect observation of the robot motion, the presence of the uncertainty makes it challenging to enforce the safe set $\mathcal{H}$ forward invariant, or even impossible when the motion noise is unknown. In this paper, we consider the realistic situation where the robot system dynamics is stochastic and has the unknown random motion noise. With that, 
we consider the discrete-time stochastic control-affine system dynamics as follows,
\begin{equation}
\label{eq: stochastic dynamics}
        \hat{x}_{k+1} = \Tilde{f}(\hat{x}_k) + \Tilde{g}(\hat{x}_k)u_k + \epsilon_k
\end{equation}
where $\hat{x}_k \in \mathbb{R}^{n}$ is the robot state and $u_k \in \mathbb{R}^m$ is the control input. $\Tilde{f}:\mathbb{R}^n \mapsto \mathbb{R}^n$ and $\Tilde{g}:\mathbb{R}^{n} \mapsto {\mathbb{R}^{n \times m}}$ are locally Lipschitz continuous. $\epsilon_k$ is the unknown random motion noise. 
Note the distribution of the noise $\epsilon_k$ is unknown and may evolve over time, which makes the safety-critical control problem more challenging. 
To satisfy the safety constraint defined by Eq.\eqref{eq: safe constraint}, we consider this problem in a chance-constrained setting. Formally, given the user-defined failure probability $\beta \in (0, 1)$ for the safety condition, we have:
\begin{equation}\label{eq: prob h}
    P(h(\hat{x}_{k+1}) \geqslant 0) \geqslant 1-\beta.
\end{equation}
where $P(\centerdot)$ is the probability under the event $(\centerdot)$. 

Given the stochastic system dynamics in Eq.(\ref{eq: stochastic dynamics}) and motivated by Probabilistic Safety Barrier Certificates in \cite{luo2020multi}, the probabilistic safety constraint in Eq. (\ref{eq: prob h}) can be satisfied by the chance constraints over control $u_k$ summarized as the following lemma~\ref{lemma: P_CBFs to P_h}

\begin{lemma}\label{lemma: P_CBFs to P_h}
[Summarized from \cite{luo2020multi}]
Given the stochastic robot system defined in Eq.~\eqref{eq: stochastic dynamics}, the robot state $\hat{x}_{k}$ is considered as a random variable at the time step $k$, and the ${\mathcal{B}}(\hat{x}_k)$ is the corresponding admissible safe control space which renders the robot's state $\hat{x}_{k+1}$ collision-free at time step $k+1$ if without noise.
Then the probabilistic constraint on control is the sufficient condition of the probabilistic safety constraint defined in Eq. \eqref{eq: prob h}, which is formulated as,
\begin{equation}\label{eq: B2h}
       P(u_k \in {\mathcal{B}}(\hat{x}_{k})) \geqslant 1-\beta \Rightarrow P(\hat{x}_{k+1} \in \mathcal{H}) \geqslant 1-\beta
    \end{equation} 
\end{lemma}

\begin{proof}
Assuming the robot is at the safe state initially and given the definition of the safe control space in Eq.(\ref{eq: safe control space}), we have $u_k \in {\mathcal{B}}(\hat{x}_{k})$  $\Rightarrow$ $\hat{x}_{k+1} \in \mathcal{H}$ and $u_k \notin {\mathcal{B}}(\hat{x}_{k})$  $\not\Rightarrow$ $\hat{x}_{k+1} \notin \mathcal{H}$. Therefore, we have $P(u_k\in {\mathcal{B}}(\hat{x}_{k})) \leqslant P(\hat{x}_{k+1} \in \mathcal{H})$ and thus Lemma \ref{lemma: P_CBFs to P_h} holds true.
\end{proof} 

The work in \cite{luo2020multi} derived the deterministic admissible control space where Eq. (\ref{eq: B2h}) holds through assuming $\epsilon_k$ in Eq.(\ref{eq: stochastic dynamics}) follows a uniform distribution known as prior. However, it is non-trivial to derive such control space when the distribution of $\epsilon_k$ is unknown and may be time-varying.

\subsection{Adaptive Conformal Prediction} \label{sec: cp}
Conformal Prediction (CP) \cite{papadopoulos2002inductive} is a statistical method to formulate a certified region for complex prediction models without making assumptions about the prediction model.

Given a dataset $\mathcal{D}=\{(X_k, Y_k)\}$ (where $k=1,\dots, d$, $X_k \in \mathbb{R}^n$, and $Y_k \in \mathbb{R}$) and any prediction model $F: X \mapsto Y$ trained from $\mathcal{D}$, e.g. linear model or deep learning model, the goal of CP is to obtain $\Bar{S} \in \mathbb{R}$ to construct a region $R(X^{*})=[F(X^{*})-\Bar{S}, F(X^{*})+\Bar{S}]$ so that $P(Y^{*} \in R(X^{*})) \geqslant 1-\alpha$, where $\alpha \in (0, 1)$ is the failure probability and $X^{*}, Y^{*}$ are the testing model input and the true 
value respectively. 
We can then define the nonconformity score $S_{k} \in \mathbb{R}^{+}$ as $S_{k} = |Y_k -  F(X_k)|$. The large nonconformity score suggests a bad prediction of $F(X_k)$. 
     
To obtain $\Bar{S}$, the nonconformity scores $S_{1}, \dots, S_{d}$ are assumed as independently and identically distributed (i.i.d.) real-valued random variables, which allows permutating element of the set $\mathcal{S} = \{ S_{1}, \dots, S_{d} \}$ (assumption of exchangeability) \cite{papadopoulos2002inductive}. 
Then, the set $\mathcal{S}$ is sorted in a non-decreasing manner expressed as $\mathcal{S} = \{S^{(1)},\dots, S^{(r)}, \dots, S^{(d)} \}$ where $S^{(r)}$ is the $(1-\alpha)$ sample quantile of $\mathcal{S}$. $r$ in $S^{(r)}$ is defined as $r:= \lceil (d+1)(1-\alpha) \rceil$ where $\lceil \cdot \rceil$ is the ceiling function. Let $\Bar{S} = S^{(r)}$, then the following probability holds \cite{papadopoulos2002inductive}, 
\begin{equation}\label{eq: preli-cp}
         P(F(X^{*})-S^{(r)} \leqslant Y^{*} \leqslant F(X^{*})+S^{(r)}) \geqslant 1-\alpha
\end{equation}

Thus, we can utilize $S^{(r)}$ to certify the uncertainty region over the predicted value of $Y^{*}$ given the prediction model $F$. It is noted that $P(F(X^{*})-S^{(r)} \leqslant Y^{*} \leqslant F(X^{*})+S^{(r)}) \in [1-\alpha, 1-\alpha+\frac{1}{1+d})$ \cite{papadopoulos2002inductive} and thus the probability in the Eq.(\ref{eq: preli-cp}) holds.

\textbf{Adaptive Conformal Prediction (ACP).} CP relies on the assumption of exchangeability \cite{papadopoulos2002inductive} about the elements in the set $\mathcal{S}$. However, this assumption may be violated easily in dynamical time series prediction tasks \cite{gibbs2021adaptive}.
To achieve reliable dynamical time series prediction in a long-time horizon,
the notion of ACP\cite{gibbs2022conformal} was proposed to re-estimate the quantile number online.
The estimating process is shown below:
\begin{small}
    \begin{equation}
    \label{eq: ACP update}
    \alpha_{k+1} = \alpha_{k} + \delta (\alpha-e_k) \; \mathrm{with}\; e_k=
    \begin{aligned}
        \begin{cases}
            0, \mathrm{if}\; S_{k} \leqslant S_k^{(r)}, \\
            1, \text{otherwise}.
        \end{cases}
    \end{aligned}
    \end{equation}
\end{small}

where $\delta$ is the user-specified learning rate. $S_k \in \mathbb{R}^{+}$ is the nonconformity score at time step k, and $S_k^{(r)}$ is the sample quantile. It is noted that if $S_{k} \leqslant S_k^{(r)}$, then $\alpha_{k+1}$ will increase to undercover the region induced by the quantile number $S_k^{(r)}$.

\subsection{Problem Statement} \label{sec: problem statement}

In this paper, we aim to 1) effectively quantify the chance constraints of the safe conditions and 2) integrate the probabilistic control constraint into the MPC framework for deriving satisfying control inputs.
The optimization problem can be formulated as follows,
\begin{align}\label{eq: MPC-CBFs-prob}
    \quad &\min_{{u}_{k:k+H-1 | k}}\,\, l_T(\hat{x}_{k+H|k}) +\sum^{H-1}_{\tau=0} l_s(\hat{x}_{k+\tau|k}, {u}_{k+\tau|k}) \\
        s.t.& \quad
            \hat{x}_{k+\tau+1|k} = \Tilde{f}(\hat{x}_{k+\tau|k}) + \Tilde{g}(\hat{x}_{k+\tau|k}){u}_{k+\tau|k} + \epsilon_{k+\tau|k}, \tag{9a} \label{eq: dynamics} \\
            & \quad P(u_{k+\tau|k} \in \mathcal{B}) \geqslant 1-\alpha, \tag{9b} \label{eq: probability}\\
            & \quad {u}_{k+\tau|k} \in [u_{min}, u_{max}],\quad\forall \tau=0,\ldots,H-1 \tag{9c}
\end{align}
where $\hat{x}_{k+\tau|k}$ represents the state vector at time step $k+\tau$ predicted at the time step $k$ from the current state $\hat{x}_{k}$, by applying the control input sequence ${u}_{k:k+H-1 | k}$ to the system dynamics Eq.\eqref{eq: dynamics}. $\epsilon_{k+\tau|k}$ is the motion noise at time step $k+\tau$ estimated 
at the time step $k$. $H$ is the time horizon and $l_T(\hat{x}_{k+H|k}), l_s(\hat{x}_{k+\tau|k}, {u}_{k+\tau|k})$ are the terminal cost and stage cost respectively.
$P(u_{k+\tau|k} \in \mathcal{B}) \geqslant 1-\alpha$ represents the probabilistic control constraint in this paper, and $u_{min}, u_{max}$ represent the control bounds.

It is noted that the control constraints in Eq.(\ref{eq: MPC-CBFs-prob}b) are difficult to address due to the unknown distribution-free motion noise $\epsilon_k$ in the actual stochastic dynamical system in Eq.~\ref{eq: stochastic dynamics}. 
We will propose a method to obtain the deterministic control constraints satisfying the chance constraints 
Eq.(\ref{eq: MPC-CBFs-prob}b) in the Section \ref{sec: method and alg}.

\section{Method and algorithms} \label{sec: method and alg}
\subsection{Conformal Prediction for Safety Barrier Certificates} \label{sec: ACP-SBC}

Inspired by the safety barrier certificates (SBC) in \cite{wang2017safety} with Lemma~\ref{def: CBFs}, we define the following terms.
   \begin{equation}\label{eq: bk def}
       B_k(x_k, u_k) = L_fh(x_k)+L_gh(x_k)u_k + \mathcal{K}(h(x_k))
   \end{equation}
   \begin{equation}\label{eq: hat bk def}
       \hat{B}_k (\hat{x}_k, u_k) = L_fh(\hat{x}_k) + L_gh(\hat{x}_k)u_k + \mathcal{K} (h(\hat{x}_k))
   \end{equation}

\begin{rem}
  $B_k(x_k, u_k)$ is considered as a prediction function similar to the function $F(X)$ in section \ref{sec: cp}, which is a reasonable assumption proved in Proposition \ref{prop: exist of CBFs}. Note that although the continuous-time CBF is adopted here with $L_fh,L_gh$ to compute $B_k,\hat{B}_k$, they can be transformed into the discrete forms as done in \cite{breeden2021control, luo2022sample}. 
\end{rem}

Given the Eq.(\ref{eq: bk def}) and Eq.(\ref{eq: hat bk def}), the nonconformity score between $\hat{B}_k$ and $B_k$ at global time step $k$ is formulated as $E_{B_k} = |\hat{B}_k - B_k|\in \mathbb{R}^{+}$.
$E_{B_k}$ will be stored in a non-decreasing order to construct the non-conformity score set $\mathfrak{E}$. The $1-\alpha$ sample quantile of the set $\mathfrak{E}$ is expressed as $E_{B_k}^{(r)}$ with $r= \lceil k(1-\alpha) \rceil$ where $\lceil \cdot \rceil$ is the ceiling function. Then, the following holds given Eq.(\ref{eq: preli-cp}),
    \begin{equation} \label{eq: B_hat-B_k}
        P(| \hat{B}_k - B_k| \leqslant E_{B_k}^{(r)}) \geqslant 1-\alpha
    \end{equation}

\begin{Theorem}\label{them: ACP-SBC}
Given the quantile number $E^{(r)}_{B_k}$ from the set $\mathfrak{E}$ and the failure probability $\alpha \in (0, 1)$, the state-dependent control space $\mathcal{\Tilde{B}}(\hat{x}_k)$ defined as follows  
can lead to the same probability guarantee of safety constraint, i.e. $P(h(\hat{x}_{k+1}) \geqslant 0) \geqslant 1-\alpha$.
    \begin{equation}\label{eq: def hat_B}
    \mathcal{\Tilde{B}}(\hat{x}_k) = \left \{u_k \in \mathcal{U} | B_k - E^{(r)}_{B_k} \geqslant 0 \right\}
    \end{equation}
\end{Theorem}

\begin{proof}
    We have $P(| \hat{B}_k - B_k| \leqslant E_{B_k}^{(r)}) \geqslant 1-\alpha$. Given $| \hat{B}_k - B_k| \leqslant E_{B_k}^{(r)}$, the following formulation holds: $-E_{B_k}^{(r)} \leqslant \hat{B}_k - B_k \leqslant E_{B_k}^{(r)}$, $ B_k -E_{B_k}^{(r)} \leqslant \hat{B}_k \leqslant B_k + E_{B_k}^{(r)}$. Since $B_k \geqslant 0$ is guaranteed by the CBFs and $E_{B_k}^{(r)} \in \mathbb{R}^{+}$, this equation holds only by
    \begin{equation} \label{eq: hat_B compt}
       B_k -E_{B_k}^{(r)} = L_fh(x_k) +L_gh(x_k)u_k + \mathcal{K}(h(x_k)) -E_{B_k}^{(r)} \geqslant 0. 
    \end{equation}
     
    With Lemma \ref{lemma: P_CBFs to P_h}, it yields $P(h(\hat{x}_{k+1}) \geqslant 0) \geqslant 1-\alpha$.
\end{proof}

Solving the chance constraint embedded in MPC is a challenging task since (a) it is non-trivial to quantify the time-varying noise and (b) $\hat{B}_k$ at future time steps is difficult to obtain. Theorem \ref{them: ACP-SBC} transfers the chance constraint to a deterministic constraint to make MPC solvable. We just consider the one-step ($\tau=1$) SBC for MPC horizon for analysis convenience. The multi-step SBC for MPC is analyzed in \ref{sec: ACP}.

\begin{prop}\label{prop: exist of CBFs}
     Given the robot state $\hat{x}_k\in \mathbb{R}^n$ generated by Eq.(\ref{eq: stochastic dynamics}), let the non-conformity score of the robot state expressed by $E_{x_k} = ||\hat{x}_k - x_k||$. 
     Then $E_{B_k} = |\hat{B}_k - B_k| \leqslant \mathcal{L}||\hat{x}_k - x_k||= \mathcal{L} E_{x_k}$ 
     holds where $\mathcal{L}=(\mathcal{L}_{L_fh} + \mathcal{L}_{L_gh}||u_k|| + \mathcal{L}_{\mathcal{K}h})  $ with $\mathcal{L}_{(\cdot)}$ as the Lipschitz constant under the function $(\cdot)$. This validates the existence of Eq.(\ref{eq: B_hat-B_k}). 
\end{prop}

\begin{proof}
With $B_k,\hat{B}_k$ defined in Eq.(\ref{eq: bk def}) and (\ref{eq: hat bk def}), the following derivation holds:
\begin{align}
    |\hat{B}_k-B_k| = &| L_{f}h(\hat{x}_k) + L_{g}h(\hat{x}_k) u_k + \mathcal{K}(h(\hat{x}_k)) \notag \\
    &- L_{f}h(x_k) - L_{g}h(x_k) u_k - \mathcal{K}(h(x_k)) | \notag \\
= &| L_{f}h(\hat{x}_k) - L_{f}h(x_k) + L_{g}h(\hat{x}_k) u_k \notag \\
&- L_{g}h(x_k) u_k + \mathcal{K}(h(\hat{x}_k)) - \mathcal{K}(h(x_k)) | \notag \\
\leqslant &(\mathcal{L}_{L_fh} + \mathcal{L}_{L_gh}||u_k|| + \mathcal{L}_{\mathcal{K}h}) | |\hat{x}_k - x_k|| \notag \\
= &\mathcal{L} E_{x_k} \notag .
\end{align}
\end{proof}

Proposition \ref{prop: exist of CBFs} not only shows the existence of Eq.(\ref{eq: B_hat-B_k}), but also allows control constraints in Eq.(\ref{eq: hat_B compt}) to be integrated into the MPC framework. Recall that the robot may move under the motion noise with time-varying distribution, which could easily break the i.i.d. assumption of the nonconformity score set. 
Through Proposition \ref{prop: exist of CBFs}, $\mathcal{L}$ connects $E_{B_k}$ and $E_{x_k}$, and thus if the set constructed from the element $E_{x_k}$ is not i.i.d., then the set $\mathfrak{E}$ is not i.i.d.. Therefore, it is reasonable to adopt ACP to quantify Eq.(\ref{eq: B_hat-B_k}).

\subsection{Adaptive Conformal Prediction for MPC}\label{sec: ACP}
The safety constraint in the optimization problem (Eq.\eqref{eq: MPC-CBFs-prob}) is computed by Eq.(\ref{eq: hat_B compt}) in Theorem \ref{them: ACP-SBC}. However, $E_{B^{\tau}_k} = |\hat{B}^{\tau}_k - B^{\tau}_k|$ (where $\tau$ is the MPC horizon time step) can not be obtained because we can not acquire $\hat{B}^{\tau}_k$ at the future MPC horizon time step under unknown motion noise. Hence, we cannot compute the $\alpha_k$ in Eq.(\ref{eq: ACP update}). We will adopt the time-lagged method from \cite{dixit2023adaptive} to tackle this problem.

\begin{figure}[t]	
	\centering	
	\includegraphics[scale=1, width=\linewidth]{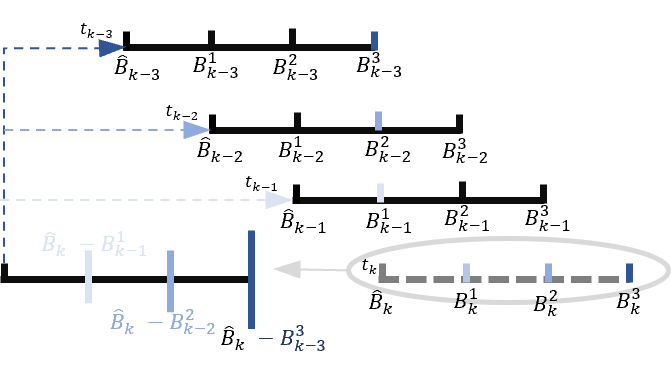}	
	\caption{Computation steps for time-lagged ACP. In the current time horizon ($H=3$) starting at time step $k$, the time-lagged error of CBFs can be computed by $E^{\tau}_{B_k} = | \hat{B}_k - B^{\tau}_{k-\tau}|$. Then, $E^{\tau}_{B_k}$, $\tau=1, \dots, H$, are stored in the nonconformity set $\mathfrak{E}$.
 }
	\label{Fig: ACP}
\end{figure}

The method of ACP for MPC in \cite{dixit2023adaptive} utilized $\tau$ step-ahead nonconformity score to represent the future nonconformity score, which is named the time-lagged method. The computation process is presented in Fig.\ref{Fig: ACP}.

However, the time-lagged method does not provide a certain probability coverage guarantee, and the probability coverage of the safety constraint throughout the entire MPC horizon is still not assured. The theoretic analysis of these two issues were provided by \cite{dixit2023adaptive}, which is summarized as follows.

\begin{lemma}
\label{lemma: probability guarantee for ACP}
    \textbf{ACP Probability Guarantee.} Given the learning rate $\delta$, the initial failure probability $\alpha_0 \in (0,1)$, and the MPC's horizon $H$. Eq.(\ref{eq: B_hat-B_k}) under MPC's horizon is bounded by,

    $1-\alpha-p_1 \leqslant \frac{1}{H} \sum_{k=1}^{H} P(E_{B_k} \leqslant E_{B_k}^{(r)}) \leqslant 1-\alpha+p_2$,
    where $p_1=\frac{\alpha_0 + \delta}{H\delta}$, $p_2=\frac{1-\alpha_0 + \delta}{H\delta}$ so that $\lim_{H\to\infty} p_1=0$  and $\lim_{H\to\infty} p_2=0$
\end{lemma}
 
\begin{prop}\label{prop: cp prob guarantee under MPC}
    \textbf{Probability Guarantee under MPC.} The probability of the safe constraint under MPC will converge to $1-\alpha-p$,
    \begin{equation}
        \frac{1}{H}\sum_{k=1}^{H} P(h(\hat{x}_{k+1})\geqslant 0) \geqslant 1-\alpha-p
    \end{equation}
    where $p=(\alpha + \delta) / (H\delta)$ is a constant related to the learning rate $\delta$, the MPC horizon $H$. When the MPC horizon $H \rightarrow \infty$, the algorithm will converge to $1-\alpha$. The proof can be found in \cite{dixit2023adaptive}.
\end{prop}

 \subsection{Algorithm Analysis}
\begin{algorithm}[t]
    \caption{ACP-SBC}\label{alg: ACP-SBC}
    \KwIn{Parameters: ACP failure probability $\alpha$ and learning rate $\delta$, prediction horizon $H$, $\gamma$ in CBFs, total time steps $tn$, time step $ts$, initial safe state $x_0$ and goal state $x_g$}
    \KwOut{Safe path of the robot }
    \BlankLine
    
    Initialization: $x_0$\;
    \While{$k \leqslant tn$}{
    \ForEach{$\tau < H$}{
            $B_k^{\tau} \gets  L_fh(x_k^{\tau})+L_gh(x_k^{\tau})u_k^{\tau} + \gamma h(x_k^{\tau})$ \;
            $E_{B_k^\tau} \gets |\hat{B}_k - B_{k-\tau}^{\tau} |$ \;
            $\mathfrak{E} \gets E_{B_k^\tau}$ \;
            $E_{B_k^{\tau}}^{(r)} \gets (\lceil k(1-\alpha) \rceil)^{\mathrm{th}}$ of $\mathfrak{E}$\;
            $\hat{B}_k^{\tau} \gets B_k^\tau -  E_{B_k^{\tau}}^{(r)}$ \;
            \eIf{ $E_{B_k^\tau} \leqslant E_{B_k^{\tau}}^{(r)} $ } {
            $\alpha_{k+1}^{\tau} = \alpha_{k}^\tau + \delta \alpha$ \;
            }{
            $\alpha_{k+1}^\tau = \alpha_{k}^\tau + \delta (\alpha -1)$ \;
            }
            }
            $u_k \gets \mathrm{MPC}(x_k^{\tau}, u_k^{\tau}, \hat{B}_k^{\tau}) \: \tau = 0,\dots, H$ \;
            $x_{k+1} \gets f(x_k) + g(x_k)u_k + \epsilon_k$ \;
    }
\end{algorithm}
The adaptive conformal prediction for safety barrier certificates named ACP-SBC is shown in Algorithm \ref{alg: ACP-SBC}. The ACP is added to quantify the uncertainty of the chance constraint within MPC and the quantified constraint is then embedded in the optimization Eq.(\ref{eq: MPC-CBFs-prob}b). The existence of the chance constraint is proved in Theorem \ref{them: ACP-SBC} and Proposition \ref{prop: cp prob guarantee under MPC} illustrates the probabilistic guarantee of this framework.

\subsection{Case for Multi-Robot Collision Avoidance}
In this subsection, we formulate a specific form of the CBFs under the safety constraint in the robot-robot case and the robot-obstacle case in order to validate ACP-SBC's performance in the simulation experiments.

\textbf{Robot-Robot Case:} the safe constraint between two robots can be defined below,
\begin{equation}
\label{eq: safety function}
    h_{i,j}({x_k}) = || {x}^{i}_k - {x}^{j}_k ||^2 - (R^i + R^j)^2
\end{equation}
where $R^i,R^j$ represent the radius of the robots $i,j$, and ${x}^{i}_k$ and ${x}^{j}_k$ are the positions of the robot $i, j$ respectively. Then, according to 
Eq.(\ref{eq: def hat_B}), the corresponding constraint becomes,
\begin{small}
    \begin{equation}
\label{eq: CBFs-multi-robot}
    2({x}^{i}_k - {x}^{j}_k) (\triangle f^{ij}_k +\triangle g^{i,j}_k u^{ij}_k) + \gamma h({x^{ij}_k}) -\mathcal{L} E^{(r)}_{B_k} \geqslant 0
\end{equation}
\end{small}

where $\triangle f^{ij}_k = f({x}^{i}_k) - f({x}^{j}_k)$, $\triangle g^{i,j}_k u^{ij}_k) = g(x^i_k) u^i_k - g(x^j_k) u^j_k$. 

\textbf{Robot-Obstacle Case:} the safe constraint between the robot and the obstacle can be defined as the following similar to the robot-robot constraints,
\begin{equation}
    h_{i, \mathrm{o}^j}({x_k}) = \left \| {x}^{i}_k - x^{j}_\mathrm{o} \right \|^2 - (R^i+R^j_\mathrm{o})^2
\end{equation}
where $x^{j}_\mathrm{o}$ represents the position of the obstacle $j$, which is regarded as a circle with the radius $R^j_\mathrm{o}$. Then we have,
\begin{equation}
    2({x}^{i}_k - x^{j}_{\mathrm{o}}) (f_i({x}^{i}_k) +g(x^{i}_k)u^{\mathrm{o}^j, i}_k) + \gamma h({x^{i}_k}) - \mathcal{L} E^{(r)}_{B_k} \geqslant 0
\end{equation}

\textbf{Multi-Robot with Multi-Obstacle Case.}
It is noted that the control policy of the multi-robot with the obstacle can be acquired by the union set between the multi-robot safe policy and the robot-obstacle policy, which can be expressed as $u^{s}_k=u^{i,j}_k \bigcap u^{i, \mathrm{o}}_k, \forall j \neq i$. 

\section{Simulation}\label{Sec: simulation}
The numerical simulations are conducted in this section to validate the effectiveness of our proposed algorithm. We first examine the impact of the specific parameters in the Algorithm \ref{alg: ACP-SBC} on the performance using a robot modeled by single integrator dynamics. Then, the simulation performance of the proposed ACP-SBC under variant noise distributions is discussed. Finally, the discussion is extended to the multi-robot collision avoidance task where behaviors of robots with both single integrator dynamics and unicycle dynamics are investigated.

\subsection{Simulation Setup}
To validate the performance of Algorithm \ref{alg: ACP-SBC}, the stochastic dynamics of the single-integrator model is presented below: 
\begin{equation}\label{eq: integrator dynamics}
    \dot{x}=u+\epsilon
\end{equation}
where $\epsilon$ is the motion uncertainty.

For the numerical simulations involved in the parameters analysis and performance comparison under variant noise distributions, 
the control input is limited by $-1 \leqslant u \leqslant 1$ and the stepsize is $ts=0.05$.

To validate the performance of Algorithm~\ref{alg: ACP-SBC}, the time horizon $H$ of MPC, $\gamma$ in the deterministic control constraint, and the failure probability $\alpha$ in the ACP, will be analyzed to guide the following simulation. Gaussian noise $\epsilon \sim \mathcal{N}(0, 1)$ is applied to the numerical simulation of the stochastic dynamics in Eq.(\ref{eq: integrator dynamics}) and the robot radius is $R=0$ for visualization convenience. 

\begin{figure*}[htbp]	
	\centering
	\subfigure[Horizon $H$]
	{
		\begin{minipage}[t]{0.21\linewidth}
			\centering
			\includegraphics[scale=1, width=0.93\linewidth]{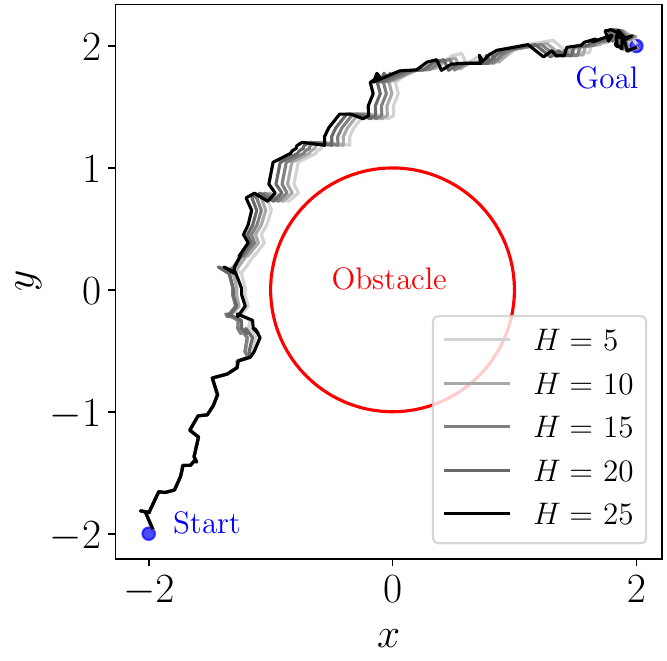}
		\end{minipage}	
	}
	\subfigure[$\gamma$ in CBFs]
	{
		\begin{minipage}[t]{0.22\linewidth}
			\centering
			\includegraphics[scale=1, width=\linewidth]{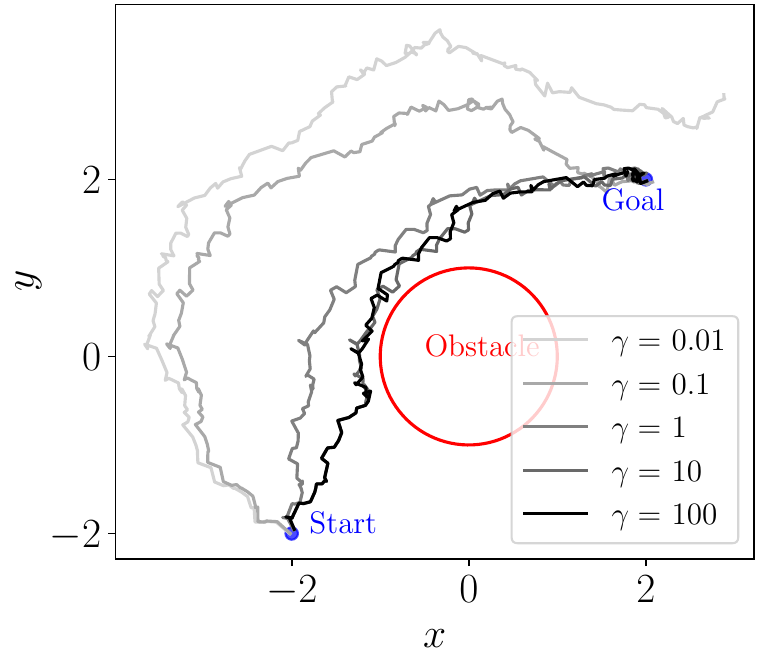}
		\end{minipage}	
	}
	\subfigure[$\alpha$ in ACP]
	{
		\begin{minipage}[t]{0.21\linewidth}
			\centering
			\includegraphics[scale=1, width=0.93\linewidth]{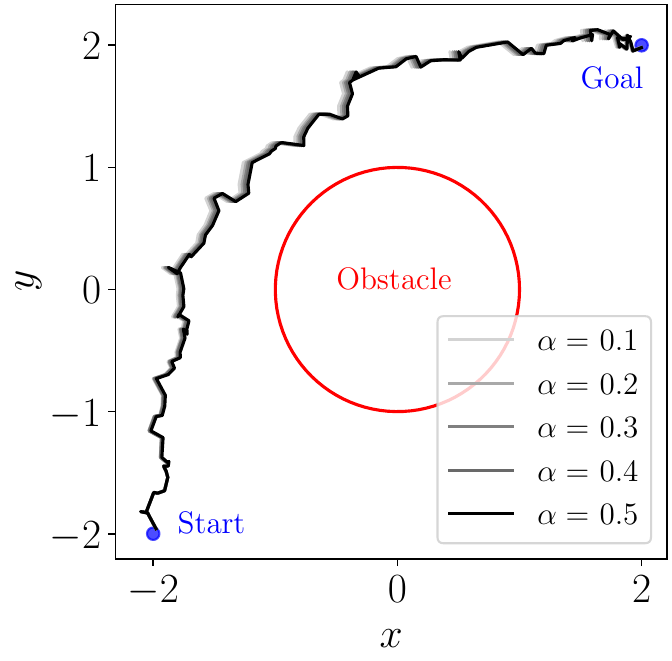}
		\end{minipage}	
	}
    \subfigure[$\hat{B}$ estimated by ACP-SBC]
	{
		\begin{minipage}[t]{0.23\linewidth}
			\centering
			\includegraphics[scale=1, width=\linewidth]{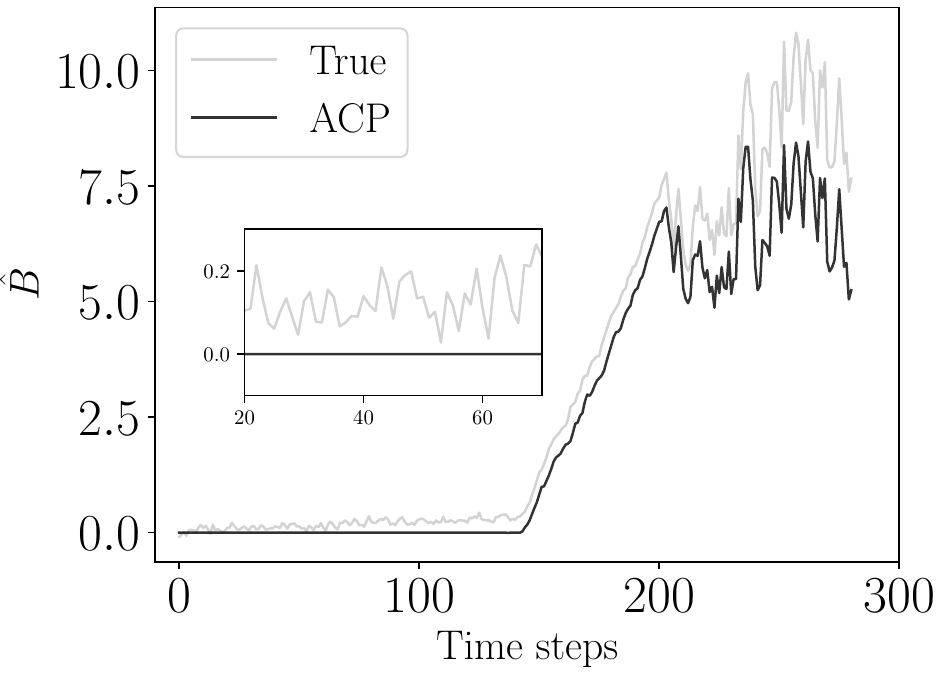}
		\end{minipage}	
	}
	\caption{The effect of the parameter in our method (ACP-SBC).}
	\label{fig: parameter}
\end{figure*}

\textbf{Parameter Analysis.} The numerical simulation result as shown in Fig.\ref{fig: parameter} for the parameter analysis illustrates that: the robot is more conservative (far away from the obstacle) when $H$ increases, $\gamma$ decreases, and $\alpha$ decreases. The result indicates that the big horizon $H$ shows more potential to avoid the obstacle because of the longer time horizon considered by the robot to get feedback from the environment. Similar to $H$, $\alpha$ decreases so the belief interval of the robot state enlarges, leading to the preservation of the robot. In the following simulation, $H=8$, $\alpha=0.05$ and $\gamma=1$ will be adopted without extra specification.

\textbf{ACP Analysis.} 
To verify the efficacy of the proposed algorithm, 
a comparative analysis between the ground truth and the $\hat{B}$ estimated by ACP-SBC is conducted using the same simulation environment designated for parameter analysis. The findings, illustrated in Fig.\ref{fig: parameter}(d), suggest that the ACP-SBC is capable of accurately estimating $\hat{B}$ with a confidence level of 95\% ($\alpha=0.05$).

\subsection{Distribution Comparision}\label{Sec: simulation distribution}
\begin{figure*}[t]
	\centering
	\subfigure[Noise from normal distribution]
	{
		\centering
		\includegraphics[scale=1, width=0.30\linewidth]{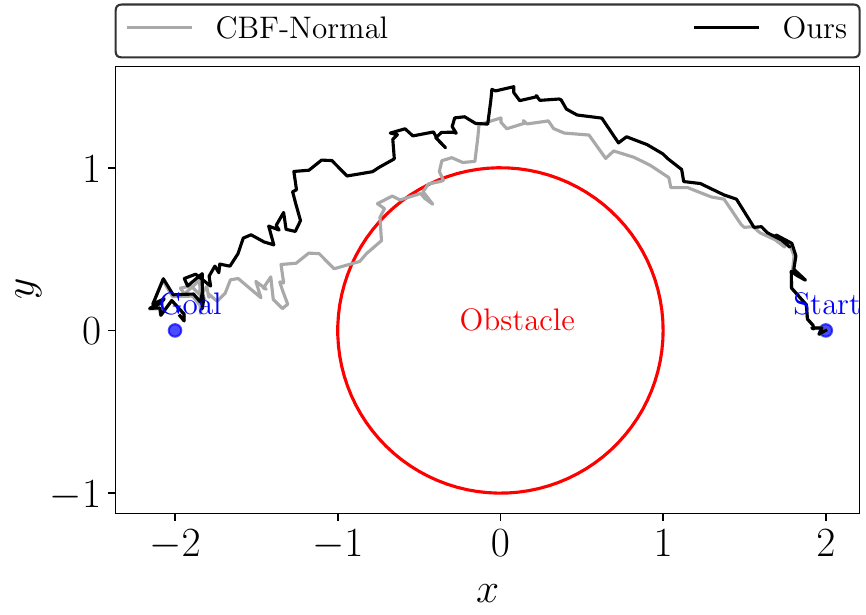}
	}
	\subfigure[Noise from unifrom distribution]
	{
			\centering
			\includegraphics[scale=1, width=0.31\linewidth]{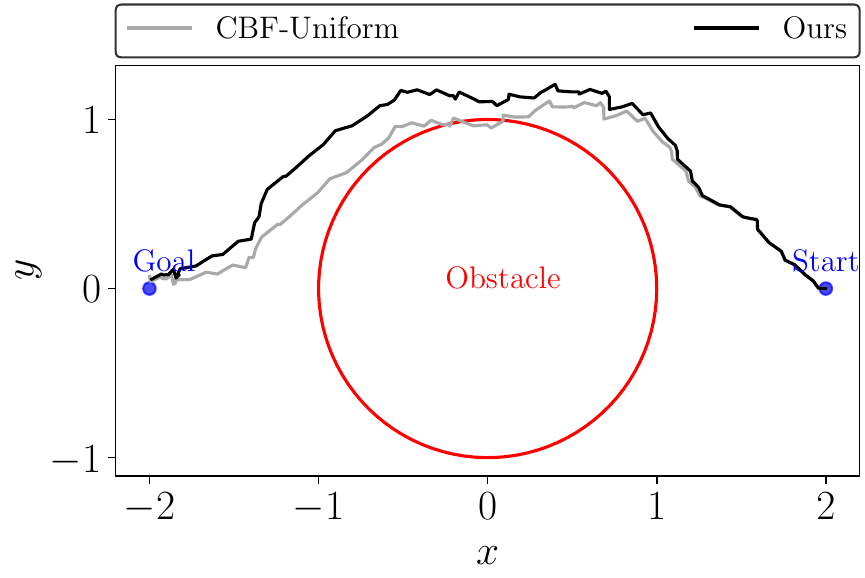}
	}
   \subfigure[Noise from mixture distribution]
	{
			\centering
			\includegraphics[scale=1, width=0.27\linewidth]{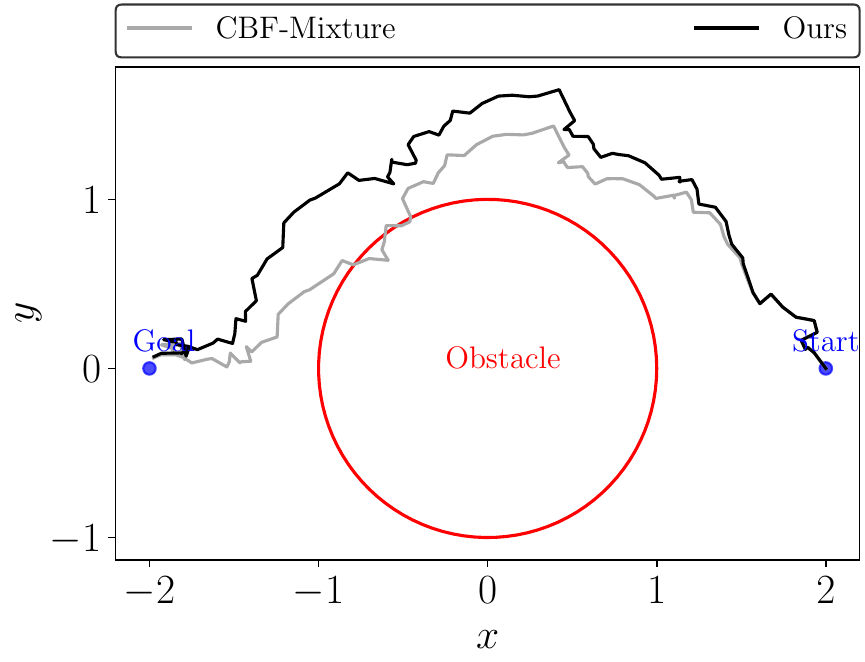}
	}
	\caption{Distribution-free validation. Normal (Uniform) means that the Gaussian distribution (uniform distribution) is added to the CBFs-based method or our method (ACP-SBC) while mixture means that one of them is randomly added to the stochastic system at each time step.}
	\label{Fig: distribution}
\end{figure*}

To validate the effectiveness of the proposed method, the CBFs-based method and ACP-SBC are applied to the single integrator simulation rendered by the Eq.(\ref{eq: integrator dynamics}). Specifically, we consider the noise following the Gaussian distribution ($\epsilon \sim \mathcal{N}(0,1)$), the uniform distribution ($\epsilon \sim U(-1, 1)$), and a random combination of the two distributions. To present the numerical simulation result clearly, the robot radius $R=0$ is also applied, which can be convenient for checking the collision with the obstacle. 

The simulation result in Fig.\ref{Fig: distribution} shows that the path generated from ACP-SBC is collision-free under different distribution inputs while all the paths collide with the obstacle using the CBFs-based method. Simulation results reveal that ACP-SBC can be applied to quantify different distributions and even the time-varying distribution.

\subsection{Multi-Robot Collision Avoidance with Obstacle}
\textbf{Multi-Robot Case}. 
We continue to verify ACP-SBC in a multi-robot scenario. First, the simulation in an obstacle-free scenario with the introduction of 30 robots modeled by single integrator dynamics is conducted. All the parameters are the same as aforementioned except the robot radius ($R=0.075$).

To illustrate the results, the minimum distance is calculated at each time step, both among robots and between a robot and an obstacle. The minimum distance below the reference distance means that collision happens. The reference distance is set to 0 because it represents the difference between the distance of two robots and the safety distance.

The simulation results in Fig.\ref{Fig: integrator simulation obs-free}(a)-(c) present that the robot still collides using CBFs, whereas ACP-SBC guarantees that the robot avoids collisions in scenarios with up to 30 robots. The simulation involved scaling the number of robots and was repeated 20 times with 20 different random seeds. The simulation result in Fig.\ref{Fig: integrator simulation obs-free}(d) validates that ACP-SBC is effective when scaling up the number of robots.

\textbf{Multi-Robot Collision Avoidance with Obstacle}. However, all the simulations so far are based on robots with single integrator dynamics. To validate the generalization of ACP-SBC, an additional experiment based on robots with unicycle dynamics is conducted on the CoppeliaSim \cite{6696520}. 

The robot's radius is $R=0.075m$ and the radius of the obstacle is $R_{\mathrm{o}} = 0.2m$. While the control input limit, $-0.08 m/s \leqslant u \leqslant 0.08 m/s$,  will be adopted in the unicycle simulation. The noise $\epsilon_{v} \sim 0.01\times\mathcal{N}(0,1)$ is applied to the linear velocity of the unicycle robot, while $\epsilon_{\omega} \sim 0.001\times\mathcal{N}(0,1)$ is allpied to the angular velocity of the unicycle robot. $H=5$ is employed and step size in the simulation platform is $ts=0.05$. $\gamma=1000$ is applied because big $\gamma$ reveals the less conservative motion, which has more potential to collide. Furthermore, ACP-SBC can still work on the small $\gamma$ through the previous simulation result.

The simulation results using robots with unicycle dynamics under three obstacles are shown in Fig.\ref{Fig: unicycle simulation obs}. The robot using the CBFs-based method collides under Gaussian distribution input, whereas the ACP-SBC method remains collision-free. The simulation result validates that our ACP-SBC can also be applied to the unicycle robot.

\begin{figure*}[!t]	
	\centering
	\subfigure[CBFs]
	{
		\begin{minipage}[t]{0.23\linewidth}
			\centering
			\includegraphics[scale=1, width=\linewidth]{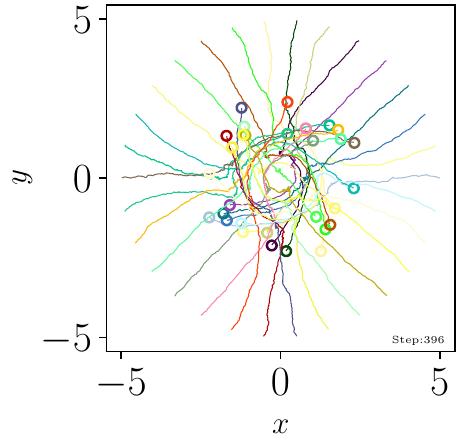}
		\end{minipage}	
	}
	\subfigure[Our method: ACP-SBC]
	{
		\begin{minipage}[t]{0.23\linewidth}
			\centering
			\includegraphics[scale=1, width=0.97\linewidth]{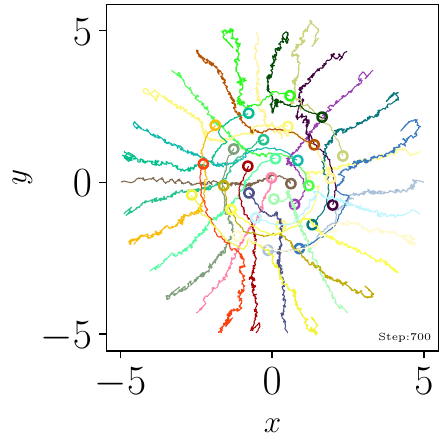}
		\end{minipage}	
	}
    \subfigure[Minimum distance]
	{
		\begin{minipage}[t]{0.23\linewidth}
			\centering
			\includegraphics[scale=1, width=\linewidth]{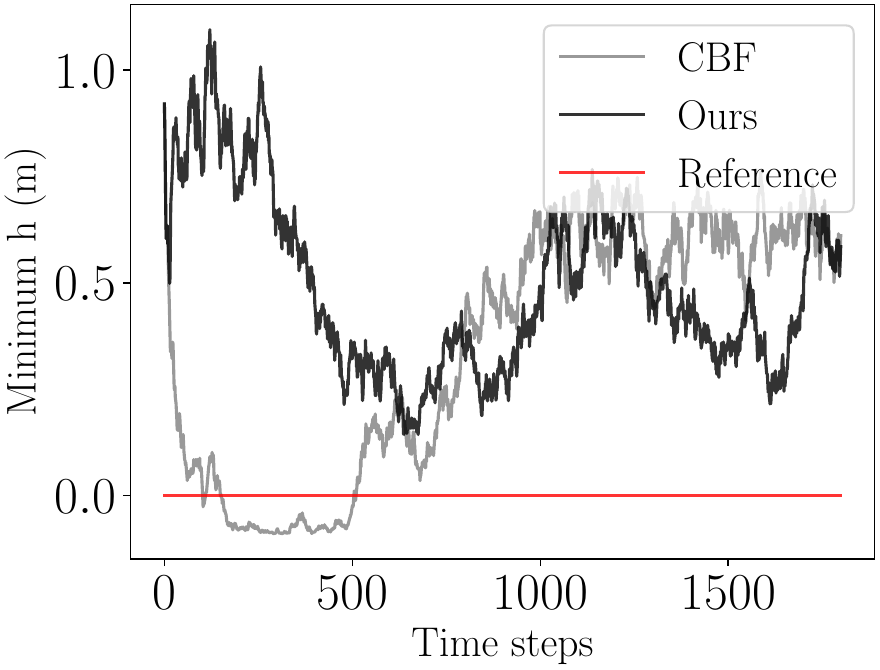}
		\end{minipage}	
	}
    \subfigure[Minimum distance of robots under 20 experiments]
	{
		\begin{minipage}[t]{0.23\linewidth}
			\centering
			\includegraphics[scale=1, width=\linewidth]{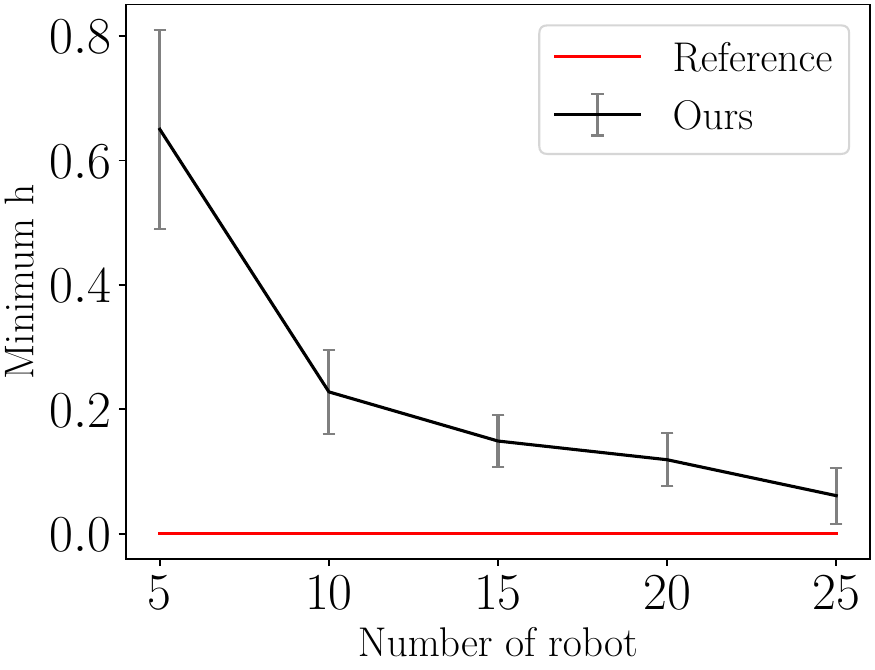}
		\end{minipage}	
	}
 
	\caption{Simulation for multi-robot coordination using 30 robots with single integrator dynamics in an obstacle-free environment.} 
	\label{Fig: integrator simulation obs-free}
\end{figure*}

\begin{figure*}[ht]	
	\centering
	\subfigure[CBFs]
	{
		\begin{minipage}[t]{0.31\linewidth} \label{fig:CBFs}
			\centering
			\includegraphics[scale=1, width=0.9\linewidth]{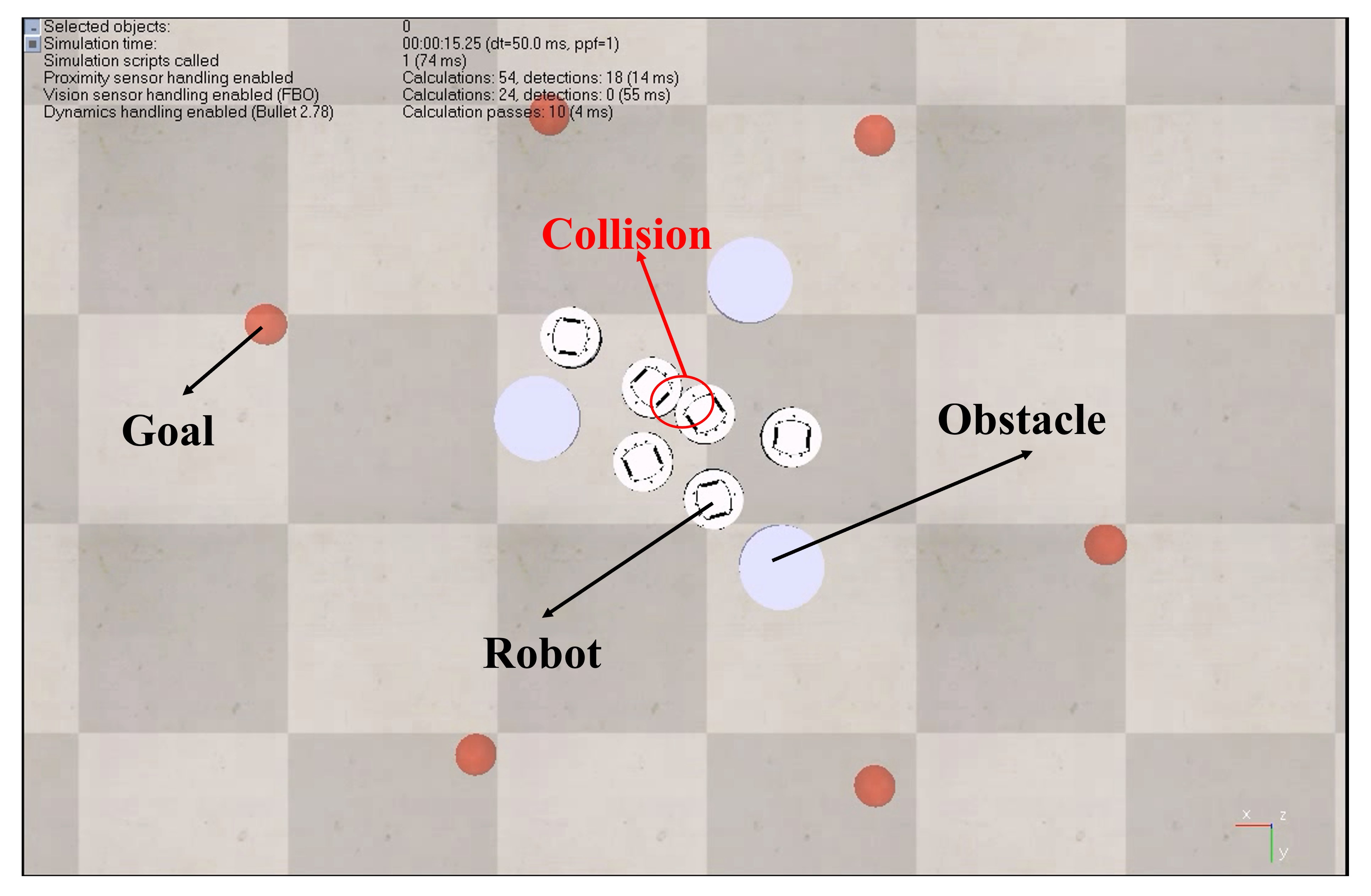}
		\end{minipage}	
	}
	\subfigure[Our method: ACP-SBC]
	{
		\begin{minipage}[t]{0.31\linewidth} \label{fig:acp-sbc}
			\centering
			\includegraphics[scale=1, width=0.9\linewidth]{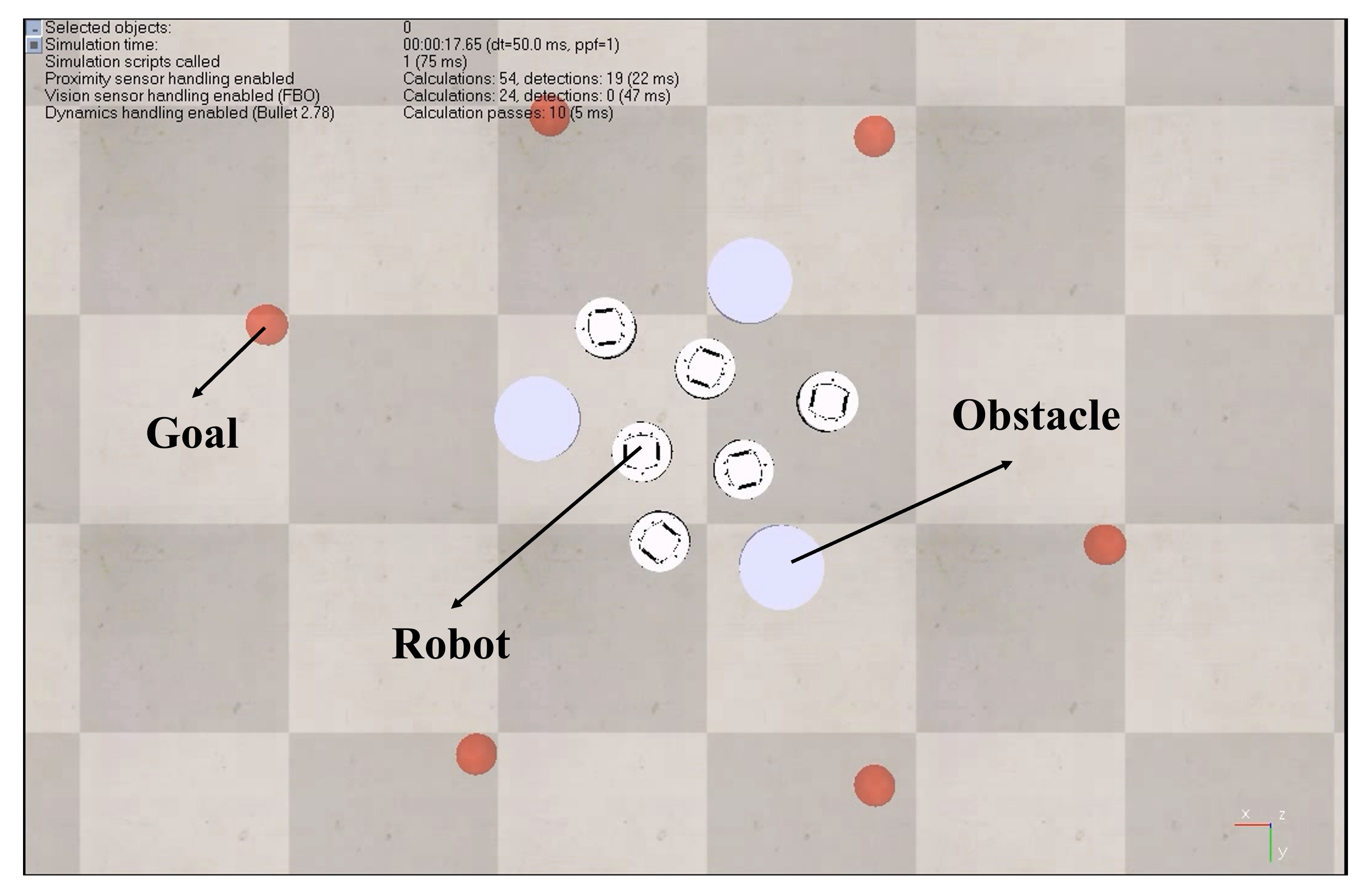}
		\end{minipage}	
	}
    \subfigure[Minimum distance]
	{
		\begin{minipage}[t]{0.31\linewidth} \label{fig:minimum distance}
			\centering
			\includegraphics[scale=1, width=0.8\linewidth]{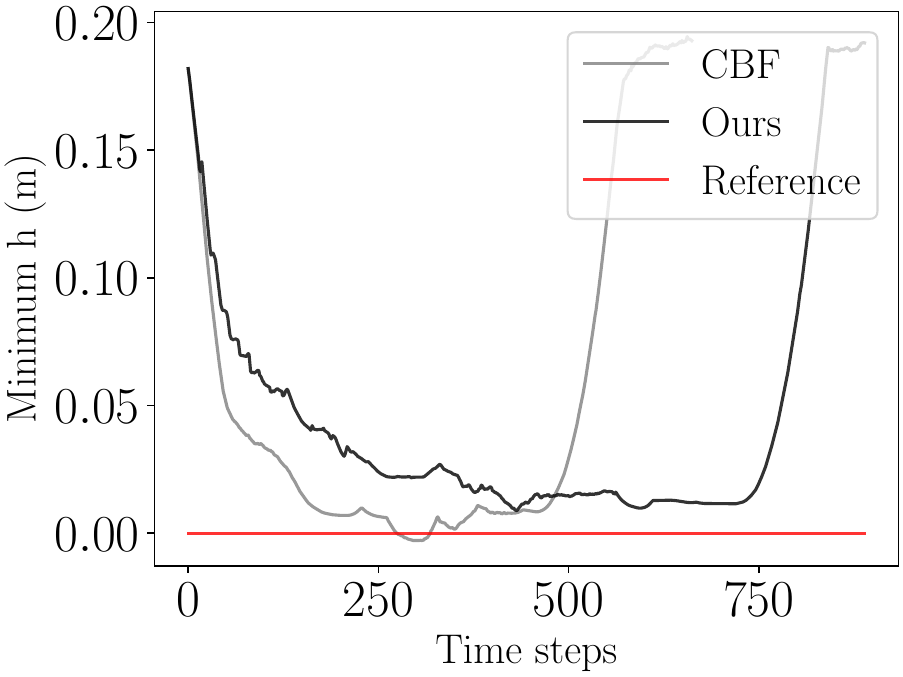}
		\end{minipage}	
	}
	\caption{Simulation with 6 Khepera IV robots with unicycle dynamics for multi-robot coordination in CoppeliaSim with Fig.\ref{fig:CBFs} using CBFs-based method and Fig.\ref{fig:acp-sbc} using ACP-SBC. It is obvious that collision happens using the CBFs-based method as shown in Fig.\ref{fig:CBFs} and the minimum distance shown in Fig.\ref{fig:minimum distance} also reveals that the collision, as the plot intersects the reference line, indicating that $h<0$}
	\label{Fig: unicycle simulation obs}
\end{figure*}

\section{Conclusion}\label{Sec: conclusion}
This paper addresses the problem of uncertainty-aware safety-critical control under stochastic system dynamics with unknown motion noise. A new framework, ACP-SBC, is proposed to quantify the uncertainty without assuming specific forms of the distribution of the motion noise or the prediction model. 
Theoretical analysis is provided to justify the performance of our proposed method. 
Simulation results validate the effectiveness of the proposed ACP-SBC framework.

\bibliographystyle{IEEEtran}
\bibliography{IEEEabrv,ref}

\end{document}